\documentclass{article}

\usepackage{times}
\usepackage{graphicx} 
\usepackage{subfigure} 

\usepackage{natbib}

\usepackage{algorithm}
\usepackage{algorithmic}

\usepackage{hyperref}

\usepackage[accepted]{icml2017}

\usepackage{amsmath}
\usepackage{amsthm}
\newtheorem{theorem}{Theorem}
\newtheorem{lemma}[theorem]{Lemma}
\usepackage[makeroom]{cancel}
\usepackage{todonotes}
\usepackage{booktabs}
\newtheorem{proposition}{Proposition}
\icmltitlerunning{Learnable Explicit Density for Continuous Latent Space and Variational Inference}

\begin{document} 
	
	\title{Learnable Explicit Density}
	
	\twocolumn[
	\icmltitle{Learnable Explicit Density for \\ Continuous Latent Space and Variational Inference}
	
	
	

	\begin{icmlauthorlist}
	\icmlauthor{Chin-Wei~Huang}{mila}
    \icmlauthor{Ahmed~Touati}{mila}
    \icmlauthor{Laurent~Dinh}{mila}
    \icmlauthor{Michal~Drozdzal}{mila,imagia}
    \icmlauthor{Mohammad~Havaei}{mila,imagia}
    \icmlauthor{Laurent~Charlin}{mila,hec}
    \icmlauthor{Aaron~Courville}{mila,cifar}
	\end{icmlauthorlist}
	
    \icmlaffiliation{mila}{MILA, Universit\'e de Montr\'eal, Canada}
    \icmlaffiliation{imagia}{Imagia Inc., Canada}
    \icmlaffiliation{hec}{HEC Montr\'eal, Canada}
    \icmlaffiliation{cifar}{CIFAR Fellow, Canada}

    \icmlcorrespondingauthor{Chin-Wei Huang}{cw.huang427@gmail.com}
    
	\icmlkeywords{boring formatting information, machine learning, ICML}
	
	\vskip 0.3in
	]
	
	
	
	\printAffiliationsAndNotice{}  
	
	\begin{abstract} 
		
In this paper, we study two aspects of the variational autoencoder (VAE): the prior distribution over the latent variables and its corresponding posterior.
		First, we decompose the learning of VAEs into layerwise density estimation, and argue that having a flexible prior is beneficial to both sample generation and inference.
		Second, we analyze the family of inverse autoregressive flows (inverse AF) and show that with further improvement, inverse AF could be used as universal approximation to any complicated posterior.
		Our analysis results in a unified approach to parameterizing a VAE, without the need to restrict ourselves to use factorial Gaussians in the latent real space. 
	\end{abstract} 
	
	\section{Introduction}

	Deep Gaussian Latent Models ~\cite{Rezende2014}, also known as Variational Autoencoders (VAEs) ~\cite{Kingma2014}, fall within the paradigm of Maximum Likelihood Estimate (MLE) and are often applied in computer vision problems. 
	However, training with MLE usually leads to overestimation of the entropy of the data distribution ~\cite{mpass}. 
	This is an undesirable property, as natural images are usually assumed to lie within a lower dimensional manifold, and the additional entropy (and other simplifying modeling assumptions for the purpose of explicit density estimation) often leads to a marginal likelihood with probability mass spread out in the data space where there is no support in the training data, which causes the blurriness of samples. 
	These observations motivate the design of more flexible, complex families of model densities.

	Since a continuous latent variable $z$ is introduced to the model, VAEs can be interpreted as an infinite mixture model $p(x) = \int_z p(x|z)p(z)dz$ where the parameters of the class conditional distribution $p(x|z)$ are functions of the latent variable $z$ (which is thought of as class here), and there are infinitely many classes.
	Such models should theoretically have enough flexibility to capture highly complex distributions such as image manifolds, but in practice it is found to be overshadowed
	by tractable density models such as autoregressive models ~\cite{VanDenOord}, or Generative Adversarial Networks (GANs) ~\cite{NIPS2014_5423} in terms of sample generation quality.

	It is believed that the relative poor performance in sample quality lies in the fact that the introduction of a latent representation requires approximate inference, as the model distribution is biased by simplifying posterior densities ~\cite{dpca}; i.e. training is achieved by maximizing the variational lower bound on the marginal log likelihood:
	\vspace{-5 pt}
	\begin{equation}
	\mathcal{L}(\theta, \phi, \pi; x) = \mathbf{E}_{q_\phi(z|x)}\left[\log \frac{ p_\theta(x|z) p_\pi(z)}{ q_\phi(z|x)}\right]
	\label{eq:elbo}
	\vspace{-5 pt}
	\end{equation}
	where subscripts $\theta,\phi,\pi$ denote the parameters of the associated distributions.
	
	We discuss two aspects of training with the bound. First, maximizing (\ref{eq:elbo}) with respect to $\phi$ amounts to minimizing $\mathcal{KL}(q_\phi(z|x)||p(z|x))$; the variational distribution, $q(z|x)$, can thus be viewed as an approximate to the true posterior, $p(z|x)$.
	Simplifying $q(z|x)$ (e.g. by using a factorial Gaussian as a common practice) is problematic, as the marginal log likelihood of interest $\log p(x)$ can only be optimized to the extent we are able to approximate the true posterior using the variational distribution. 
	This motivates a direct improvement of variational inference ~\cite{Rezende2015,Ranganath2014,Kingma2016}.

	Second, during training of the VAE, only a part of the latent space is explored. 
	When marginalizing out the input vector $x$, we recover the marginal $q(z) = \int_x q(z|x)p_\mathcal{D}(x)$, 
	where $\mathcal{D}$ indicates the true data distribution.
	When the marginal approximate posterior fails to fill up the prior as the prior-contractive term requires, one would risk sampling from untrained regions in the latent space.
	A direct and non-parametric treatment of sampling from such regions of the prior would be to take $q(z)$ as the prior, but the integral is intractable and the data distribution is only partially specified by a limited training data.
	Even if we take the empirical distribution of $p_\mathcal{D}(x)$, we would have a mixture model of up to $n$ components, where $n$ is the number of training data points, which would be impractical given the scale of modern machine learning tasks. 
	A workaround of this problem is to take a random subset of $\mathcal{D}$, or introduce a learnable set of pseudo-data of size $K$, and set the prior to be $p(z) = \sum_{j=1}^K \frac{1}{K}q(z_j|x_j)$, which is shown to be promising in the recent work done by ~\citet{Tomczak}. 
	Another approach is to directly regularize the autoencoder by matching the aggregated posterior with the prior, as in ~\citet{MakhzaniSJG15}.
	

	In this paper, we make two main contributions. First, we analyze the effect of making the prior learnable. 
	We show that training with the variational lower bound under some limit conditions matches the marginal approximate posterior with the prior, which is desirable from the generative model point of view. 
	We then decompose the lower bound, and show that updating the prior alone brings the prior closer to the marginal approximate posterior, suggesting that having the prior trainable is beneficial to both sample generation and inference. 
	Our second contribution is to prove that by using the family of inverse AF ~\cite{Kingma2016}, one can universally approximate any posterior. 
	This theoretically justifies the use of inverse AF to improve variational inference. 
	We unified the two aspects and propose to use invertible functionals ~\citet{DinhSB16} and ~\citet{Kingma2016} to parameterize explicit densities for both the prior and approximate posterior.


	\section{Marginal Matching Prior}
	\label{sec:mmp}
	
	\begin{figure}[t]
		\centering
		\includegraphics[width=0.6\linewidth]{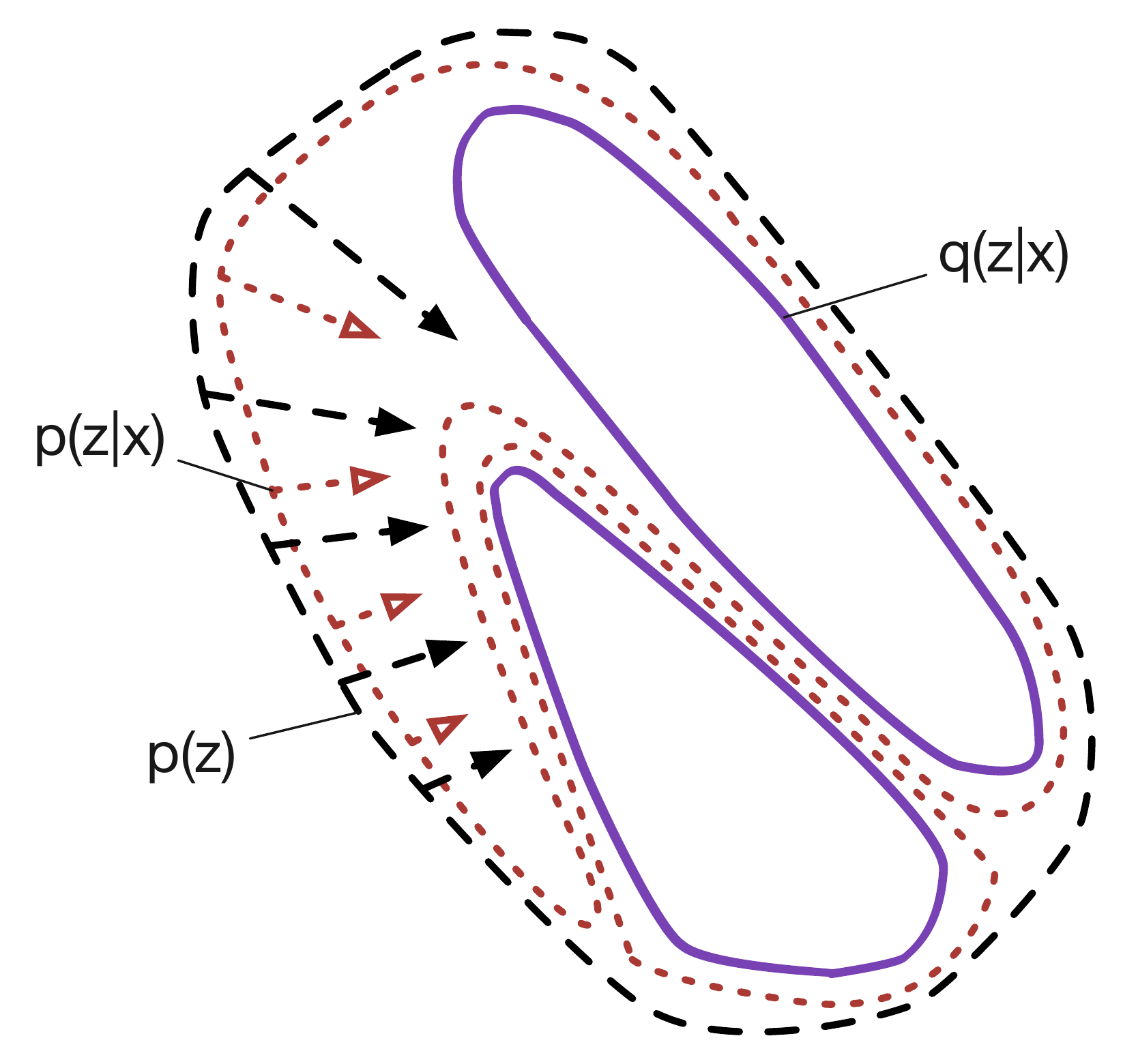}
		\caption{Effect of prior on posterior. 
			Matching the prior $p(z)$ with the marginal approximate posterior $\mathbf{E}_x[q(z|x)]$ makes the true posterior $p(z|x)$ easier to model, since it pushes the true posterior to be closer to the approximate posterior.}
		\label{fig:mmp}
		\vspace{-15pt}
	\end{figure}

	We claim that maximizing the variational lower bound explicitly matches the marginal $q(z)$ with the prior $p(z)$. 
	By decomposing the lower bound, we then suggest using a learnable prior to improve sampling, i.e. to have a prior that matches the marginal $q(z)$ instead.
	
	Let us define encoding and decoding distributions as $q(z|x)$ and $p(x|z)$ respectively, a prior as $p(z)$ and a data distribution as $p_\mathcal{D}(x)$. 
	Our goal is to train an auto-encoder as a generative model by keeping $q(z) = \int_x p_\mathcal{D}(x)q(z|x)dx$ close to the prior.
	This can be achieved at the limits of the following two conditions ~\cite{elbosurgery}:
	%
	\vspace{-1pt}
	\begin{center}
		$\textit{1. }q(z|x)\rightarrow p(z|x) \quad\forall x\sim p_\mathcal{D}(x) \quad\quad\textit{2. }p(x)\rightarrow p_\mathcal{D}(x)$
	\end{center}
	\vspace{-3pt}
	In words, given a perfect approximate posterior $q(z|x)$ of $p(z|x)$ and a perfect marginal likelihood $p(x)$ of $p_\mathcal{D}(x)$, we have the marginal $q(z)$ converge to the prior, i.e.
	%
	\begin{equation}
	\begin{split}
	q(z) &= \int_x p_\mathcal{D}(x)q_\phi(z|x)dx \\
	&\overset{\textit{1.}}{\rightarrow} \int_x p_\mathcal{D}(x)p_{\theta,\pi}(z|x)dx \\
	&\overset{\textit{2.}}{\rightarrow} \int_x p_{\theta,\pi}(x)p_{\theta,\pi}(z|x)dx = p_\pi(z)
	\end{split}
	\end{equation}
	%
	That is, to have $q(z)\rightarrow p(z)$, we need to ensure the two conditions are satisfied. 
	We can cast it as an optimization problem by minimizing the KL-divergences:
	\vspace{5pt}
	\begin{equation} 
	\begin{split}
	&\min \mathbf{E}_{p_\mathcal{D}(x)}[\mathcal{KL}(q(z|x)||p(z|x))] + \mathcal{KL}(p_\mathcal{D}(x)||p(x)) \\
	&= \max \mathbf{E}_{x\sim p_\mathcal{D}(x)}[\mathcal{L}(\theta,\phi,\pi;x)]
	\end{split}
	\label{eq:kl-elbo}
	\end{equation}
	%
	The equality is a direct result of rearrangement of terms. 
	What (\ref{eq:kl-elbo}) implies is that maximizing the variational lower bound brings us to the limit conditions under which marginal approximate posterior $q(z)$ should match the prior given enough flexibility in the assumed form of densities. 
	
	Now if we maximize (\ref{eq:kl-elbo}) with respect to $\pi$ while holding $\theta$ and $\phi$ fixed like doing coordinate ascent, the samples drawn from the doubly stochastic process $x\sim p_\mathcal{D}(x)$, $z\sim q(z|x)$ can be thought of as a projected data distribution that we want to model using the prior distribution:
	\vspace{5pt}
	\begin{equation}
    \max_\pi \mathbf{E}[\mathcal{L}] =
	\min_\pi \mathcal{KL}(\mathbf{E}_{x\sim p_\mathcal{D}(x)}[q(z|x)]||p_\pi(z))
	\label{eq:mmp}
	\vspace{5pt}
	\end{equation}

	As a result, having a learnable prior allows us to sample from the marginal approximate posterior if the above divergence metric goes to zero. 
    
	Another advantage of a learnable prior can be visualized by the cartoon plot in Figure \ref{fig:mmp}. 
	When we fix the approximate posterior and update the prior such that it becomes closer to the marginal approximate posterior, it concentrates the probability mass in such a way that the true posterior becomes closer to the approximate posterior, as $p(z|x)\propto p(z)$.
	In other words, the region of high posterior density not covered by the approximate posterior will be reduced, which effectively means our proposal as variational distribution could be improved by having a better prior which simplifies the true posterior.

	\section{Inverse Autoregressive Flows as Universal Posterior Approximator}
	\label{sec:univ}
	
	In ~\citet{Kingma2016}, a powerful family of invertible functions called the Inverse Autoregressive Flows (inverse AF or IAF) were introduced, to improve variational inference.
	It is thus of practical and fundamental importance to understand the benefits of using inverse AF and how to improve them.

	In this section, we show that normalizing flows from a base distribution (such as uniform distribution) under autoregressive assumptions are universal approximators of any density (as suggested in ~\citet{ganstutorial}), given enough capacity when a neural network is used to parameterize non-linear dependencies.

	\begin{lemma}
		\textbf{Existence of solution to a nonlinear independent component analysis problem.} Given a random vector $X = (X_i)_{i=1 \ldots m} \in \mathbf{R}^m $, there always exists a mapping $g$ from $\mathbf{R}^m$ to $\mathbf{R}^m$ such that the components of the random vector $Y=f(X)$ are statistically independent.  
		
		\label{lemma:existence}
		\vspace{-5pt}
	\end{lemma}
	\begin{proof}
		See ~\citet{Hyvarjnen1998} for the full proof.
		Here we point out that the transformation $g$ used in the proof falls within the family of autoregressive functions: $f = (f_i)_{i=1 \ldots m}$  where $y_i = f_i(x_i, y_1,\ldots, y_{i-1}) = P(x \leq x_i \,|\, y_1,...,y_{i-1})$, for $i=1 \ldots m$. $f_i$ is the conditional CDF and $Y \sim Unif([0,1]^m)$. Then any distribution of a random variable $x$ can be warped into an independent distribution via the CDFs, specifically by a kind of Gram-Schmidt process-like construction.
	\end{proof}

	\begin{proposition} \textbf{Inverse autoregressive transformation as universal approximator of any density.} Let $X$ be a random vector in an open set $\mathcal{U} \subset \mathbf{R}^m$. We assume that $X$ has a positive and continuous probability density distribution. 
		There exists a sequence of mappings $(G_{n})_{n \geq 0}$ from $(0, 1)^m$ to $\mathbf{R}^m$ parametrized by autoregressive neural networks such that the sequence $X_n = G_{n}(Y)$ where $Y \sim Unif((0, 1)^m)$ converges in distribution to X. 	
	\vspace{-5pt}
	\end{proposition}
	\begin{proof}
		We consider the mapping $f$ defined in the proof of Lemma \ref{lemma:existence}. 
		As $f$ is autoregressive, the Jacobian of $f$ is an upper triangular matrix whose diagonal entries are equal to the conditional densities which are positive by assumption.
		The determinant of the Jacobian, which is equal to the product of diagonal entries, is positive. 
		By the \textit{inverse function theorem}, f is locally invertible. 
		As f is also injective (as follows from the bijectivity of CDF), f is globally invertible and let g denotes its inverse. 
		g is an autoregressive function and by the \textit{universal approximation theorem} \cite{Cybenkot1989}, we know that there exists a sequence of  mappings $(G_{n})_{n \geq 0}$ from $(0, 1)^m$ to $\mathbf{R}^m$ parametrized by autoregressive neural networks that converge uniformly to $g$. 
		Let $X_n = G_{n}(Y)$ where $Y \sim Unif((0, 1)^m)$. Let $h$ be a real-valued bounded continuous function on $\mathbf{R}^m$.
		The latter uniform convergence implies that since $G_n$ converge pointwise to $g$, then by continuity of $h$, $h\circ G_n$ converges pointwise to $h \circ g$. As $h$ is bounded, the \textit{dominated convergence theorem} gives that $ \mathbf{E}[h(X_n)] = \mathbf{E}[h(G_n(Y))]$ converges to $\mathbf{E}[h(g(Y))] = \mathbf{E}[h(X)]$. As the latter statement is valid for all bounded continuous function $h$, $X_n$ converge to $X$ in distribution.
		\vspace{-5pt}
	\end{proof}%

	Note that $G$ is usually parameterized as an invertible function, at the expense of flexibility, to have a tractable Jacobian. 
    Special designs of such a function, other than affine transformation ~\cite{Kingma2016}, could be made to improve the flow; otherwise one would need to compose multiple layers of transformations to have a richer distribution family.
    Our proof shows that, with careful designs of approximate posteriors, VAEs could have asymptotic consistency.
	
	\section{Proposed Method}
	\label{sec:method}
	\vspace{-1pt}
	As suggested in sections \ref{sec:mmp} and \ref{sec:univ}, we propose to use one-to-one correspondence to define a learnable explicit density (LED) model for both inference and sample generation.
	First, inspired by (\ref{eq:mmp}), we found that updating the prior alone is reminiscent of MLE. 
	One can think of data points projected onto the latent space via Monte Carlo sampling as a data distribution $q_\mathcal{D}(z) = \mathbf{E}_{p_\mathcal{D}(x)}[q(z|x)]$ in space $z$.
	A unimodal prior tends to overestimate the entropy of $q_\mathcal{D}(z)$.
	A powerful family of real non-volume preserving (Real NVP) transformations ~\cite{DinhSB16} can be applied to real variables. 
	It is thus natural to incorporate Real NVP into VAEs to jointly train an explicit density model as prior. 
	We define the prior (and also the approximate posterior) with change of variable formula: $p(z)=p(z_0)|\frac{\partial h}{\partial z_0}(z_0)|^{-1}$ where $h:z_0\rightarrow z$. 
	To compute the density of the projected data distribution, we inversely ($h^{-1}$) transform the samples $z\sim q_\mathcal{D}(z)$ into the base variable $z_0$ with tractable density ~\cite{Dinh2014}. 
	We define the posterior likewise, as in ~\citet{Rezende2015}, with $g:z'\rightarrow z$. 
	Objective (\ref{eq:elbo}) can thus be modified as
	
	\vspace{-5pt}
	\begin{equation}
	\begin{split}
	\mathcal{L}= \!\begin{aligned}[t] &\mathbf{E}_{q(z'|x)}[\log p(x|g(z'))] + \\
	&\mathbf{E}_{q(z'|x)}\left[\log p(h^{-1}\circ g(z')) + \log \left| \frac{\partial h^{-1}}{\partial z}(g(z')) \right| \right] - \\
	&\mathbf{E}_{q(z'|x)}\left[\log q(z'|x) - \log \left|\frac{\partial g}{\partial z'} (z')\right| \right] 
	\end{aligned}
	\end{split}
	\label{eq:obj}
	\end{equation}
	\vspace{-5pt}%
		
	For permutation invariant latent variables, $h$ is implemented with random masks. For latent variables that preserve the spatial correlation when a convolutional network is used, we choose to use a checkerboard style mask ~\cite{DinhSB16,2016arXiv161105209A}. 
	Interestingly, sampling of such models is similar to block Gibbs sampling for energy based models (e.g. Ising models) that define the correlation between adjacent pixels. 
	
	Second, for the posterior distribution, we construct $g$ by inverse AF, which is parallelizable when combined with MADE ~\cite{Germain:2015} or PixelCNN ~\cite{VanDenOord}. 
	In fact, inverse AF can be thought of as a generalization of Real NVP, as the Jacobian of the masked operation used in Real NVP is upper triangular. 
	
	\begin{figure}
		\centering
		\includegraphics[width=1.0\linewidth]{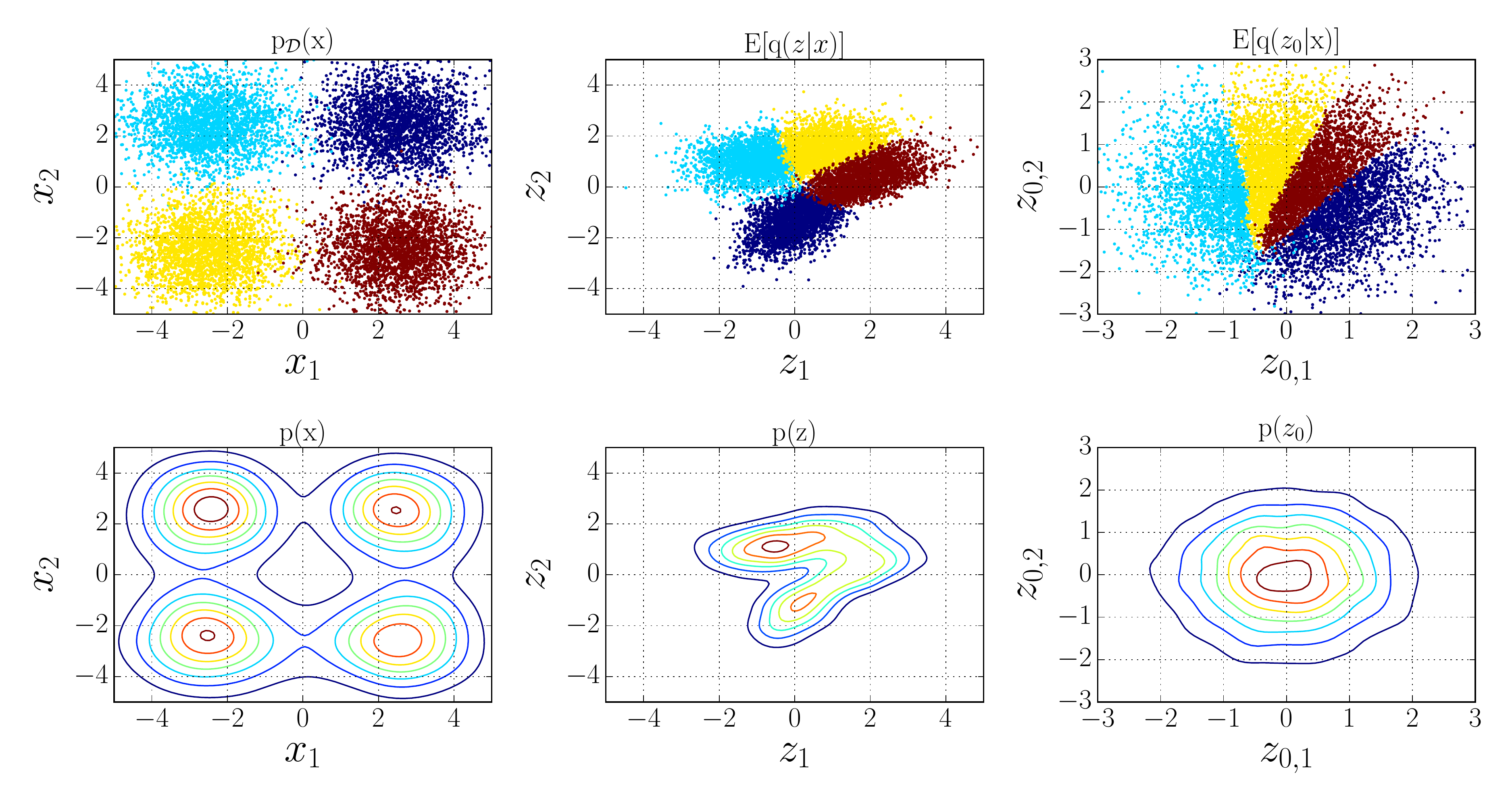}
		\caption{Fitting a Gaussian mixture distribution. $E[\cdot]$ indicates marginalization over the data $x\sim p_\mathcal{D}(x)$. 
			Clockwise from top left: projection of data distribution $p_\mathcal{D}(x)$ onto the prior space $\mathbf{E}(q(z|x))$, and the base distribution space $\mathbf{E}(q(z_0|x))$; density maps of the base distribution $p(z_0)$, the transformed prior $p(z)$ and and marginal model distribution $p(x)$.}
		\label{fig:encdec}
		\vspace{-5pt}
	\end{figure}

	\begin{table}[t]
		\centering
        \caption{Effect of increasing prior complexity. L\textsubscript{post}: number of MADE layers used for posterior. 
        Two hidden layers of 512 nodes were used for each layer of transformation.
        L\textsubscript{prior}: number of NVP layers used for prior. 
        One hidden layer of 100 nodes was used for each layer of transformation.
        For multi-layer perceptron, two hidden layers with 200 nodes were used and the dimension of the latent variable is 50. 
        Rectifier is used as non-linear activation. 
			For Residual ConvNet, we have 3 layers of residual strided convolution ~\cite{DBLP:journals/corr/HeZRS15} with [16,32,32] feature maps, using filter of size 3$\times$3. 
            Before the stochastic layer a hidden layer of 450 nodes is used. 
            The dimension of the latent variable is 32. 
            We use exponential linear units ~\cite{DBLP:journals/corr/ClevertUH15} as non-linearity. }
            \label{tab:nll_prior}
		\begin{tabular}{cccccc}
			\toprule
			\multicolumn{2}{c}{MLP} & \multicolumn{2}{c}{MLP} & \multicolumn{2}{c}{ResConv} \\
			\midrule
			L\textsubscript{post}  & NLL & L\textsubscript{prior}  & NLL & L\textsubscript{prior} & NLL\\
			\midrule
			0 & 90.78 & 0 & 90.78  & 0 & 83.11 \\
			4 & 88.89 & 4 & 88.07  & 4 & 81.87  \\
			8 & 88.71 & 8 & 87.47  & 8 & 81.70  \\
			12 & 88.70 & 12 & 86.59  & 12 & 81.44 \\
			\bottomrule
		\end{tabular}
	\vspace{-5pt}
	\end{table} 
	
	\begin{table}[t]
		\centering
        \caption{Effect of increasing both prior and posterior complexity.}
		\label{tab:nll_both}
		\begin{tabular}{ccc}
        	\toprule
        	\multicolumn{3}{c}{ResConv} \\
			\midrule
			L\textsubscript{prior} & L\textsubscript{post} & NLL \\
			\midrule
			4 NVP & 4 NVP & 81.81 \\
			8 NVP & 8 NVP & 81.55 \\
			8 NVP & 8 MADE & 80.81 \\
			16 NVP & 16 MADE & 80.60 \\
			\bottomrule
		\end{tabular}
	\vspace{-5pt}
	\end{table} 
	
	\section{Experiments}	
	\label{sec:exp}
	\vspace{-1pt}

	\textbf{Mixture of Bivariate Gaussians.} We experiment on a Gaussian mixture toy example, and visualize the effect of having a learnable prior in Figure \ref{fig:encdec}.
	During training, we observe that models with flexible prior are easier to train than models with flexible posterior. 
	Our first conjecture is that 
	to refine the posterior density, we only draw one sample of $z$ for each data point $x$, whereas refining the prior density can be viewed as modeling the projected data distribution and thus depends on as many samples as there are in the training set.
	Second, it might be due to the kind of transformation and the distance metrics that are used.
	To learn the posterior, we implicitly minimize $\mathcal{KL}(q(z|x)||p(z|x))$, which is zero forcing since samples in region that has low target density are heavily penalized.
	If $q$ begins with a sharper shape, it pays a high penalty by expansion
	to move to another mode.
	It is thus easy for the distribution to be stuck in local minima if the true posterior is multimodal, while learning the prior does not have this mode seeking problem since  the forward KL in (\ref{eq:mmp}) is zero avoiding.


	\textbf{MNIST.} We also tested our proposed method on binarized MNIST ~\cite{larochelle2011}, and report the estimated negative log likelihood as an evaluation metric.

	We compare the effects of adding more invertible transformation layers on either the prior or posterior (see Table \ref{tab:nll_prior}), or both (Table \ref{tab:nll_both}).
	From Table \ref{tab:nll_prior}, we see that models having a flexible prior easily outperform models with a flexible posterior. 
    Likelihood of a model with flexible prior can be further improved by using expressive posterior (Table \ref{tab:nll_both}) such as real NVP (81.70 $\rightarrow$ 81.55), or with MADE to introduce more autoregressive dependencies (81.55 $\rightarrow$ 80.81).

	\section{Discussion and Future Work}
	In this paper, we first reinterpret training with the variational lower bound as layer-wise density estimation. 
	Treating the Monte Carlo samples from the approximate posterior distributions as projected data distribution suggests using a flexible prior to avoid overestimate of entropy.
	We leave experiments on larger datasets and sample generation as future work.
	Second, we showed that parameterizing inverse AF using neural networks allows us to universally approximate any posterior, which theoretically justifies the use of inverse AF. 
	Our proof also implies using affine coupling law to autoregressively warp the distribution is limited.
	It is thus possible to consider designs of more flexible invertible functions to improve approximate posterior.

	\section*{Acknowledgements} 
	We thank NVIDIA for donating a DGX-1 computer used in this work.

	\nocite{langley00}
	
	\bibliography{ref}
	\bibliographystyle{icml2017}
	
\end{document}